\RequirePackage{amsmath}
\documentclass[preprint,3p,number]{elsarticle}

\usepackage[T1]{fontenc}
\usepackage[utf8]{inputenc}
\usepackage[english]{babel}
\usepackage{lmodern}
\usepackage{arydshln} 
\usepackage{tikz}
\usepackage{pgfplots}
\usepackage{listings}
\usepackage{amssymb}
\usepackage[hidelinks]{hyperref}

\usetikzlibrary{shapes}
\pgfplotsset{compat=1.14}

\newdefinition{definition}{Definition}
\newdefinition{property}{Property}
\newdefinition{problem}{Problem}
\newdefinition{example}{Example}
\newtheorem{lemma}{Lemma}
\newtheorem{theorem}{Theorem}
\newproof{proof}{Proof}

\DeclareMathOperator*{\argmax}{argmax}
\DeclareMathOperator*{\softmax}{softmax}
\renewcommand{\vec}[1]{\bar{#1}}

\newcommand{\mcc}[1]{\multicolumn{1}{c}{#1}}

\journal{Science of Computer Programming}

\begin{document}

\begin{frontmatter}
  \title{Formal Verification of Input-Output Mappings of Tree Ensembles}
  \author{John Törnblom\corref{jt}}
  \cortext[jt]{Corresponding author}
  \ead{john.tornblom@saabgroup.com}
  \address{Saab AB\\
           Bröderna ugglas gata, Linköping, Sweden}
  \author{Simin Nadjm-Tehrani\corref{snt}}
  \ead{simin.nadjm-tehrani@liu.se}
  \address{Dept. of Computer and Information Science\\
           Linköping University, Linköping, Sweden}

\begin{abstract}
Recent advances in machine learning and artificial intelligence are now being
considered in safety-critical autonomous systems where software defects may
cause severe harm to humans and the environment. Design organizations in these
domains are currently unable to provide convincing arguments that their systems
are safe to operate when machine learning algorithms are used to implement their
software.

In this paper, we present an efficient method to extract equivalence classes 
from decision trees and tree ensembles, and to formally verify that their 
input-output mappings comply with requirements. The idea is that, given that 
safety requirements can be traced to desirable properties on system input-output 
patterns, we can use positive verification outcomes in safety arguments.
This paper presents the implementation of the method in the tool VoTE (Verifier 
of Tree Ensembles), and evaluates its scalability on two case studies presented 
in current literature. 
We demonstrate that our method is practical for tree ensembles trained on 
low-dimensional data with up to 25 decision trees and tree depths of up to 20.
Our work also studies the limitations of the method with high-dimensional
data and preliminarily investigates the trade-off between large number of trees 
and time taken for verification.
\end{abstract}


\end{frontmatter}

\section{Introduction}
In recent years, artificial intelligence utilizing machine learning algorithms
has begun to outperform humans at several tasks, e.g., playing board 
games~\cite{Silver16} and diagnosing skin cancer~\cite{Esteva17}. These advances
are now being considered in safety-critical autonomous systems where software 
defects may cause severe harm to humans and the environment, e.g, airborne
collision avoidance systems~\cite{Julian16}.

Several researchers have raised concerns~\cite{Burton17,Kurd06,Russell15} 
regarding the lack of verification methods for these kinds of systems in which 
machine learning algorithms are used to train software deployed in the system.
Within the avionics sector, guidelines~\cite{DO333} describe how design
organizations may apply formal methods to the verification of safety-critical
software. Applying these methods to complex and safety-critical software is a
non-trivial task due to practical limitations in computing power and challenges
in qualifying complex verification tools. These challenges are often caused by
a high expressiveness provided by the language in which the software is
implemented in. Most research that apply formal methods to the verification of
machine learning is so far focused on the verification of neural networks, but 
there are other models that may be more appropriate when verifiability is 
important, e.g., decision trees~\cite{Breiman84}, random forests~\cite{Breiman01}
and gradient boosting machines~\cite{Friedman01}. Where neural networks are
subject to verification of correctness, they are usually adopted in
non-safety-critical cases, e.g., fairness with respect to individual
discrimination, where probabilistic guarantees are plausible~\cite{Bastani18b}.
Since our aim is to support safety arguments for digital artefacts that may
be deployed in hazardous situations, we need to rely on guarantees of absence
of misclassifications, or at least recognize when such guarantees cannot be
provided. The tree-based models provide such an opportunity since their
structural simplicity makes them easy to analyze systematically, but large
(yet simple) models may still prove hard to verify due to combinatorial
explosion.

This paper is an improved and substantially extended version of our previous
work~\cite{Tornblom19}. In that work, we developed a method to partition the 
input domain of decision trees into disjoint sets, and to explore all path 
combinations in random forests in such a way that counteracts combinatorial path
explosions. In this paper, we generalize the method to also apply to gradient 
boosting machines and implement it in a tool named VoTE. 
The paper also evaluates the tool with respect to performance. 
The VoTE source code (which includes automation of the case studies) is
published as free 
software.\footnote{https://github.com/john-tornblom/vote/tree/v0.1.1}
Compared to previous works, the contributions of this paper are as follows.
\begin{itemize}
  \item A generalization to include more tree ensembles, e.g., gradient boosting
        machines, with an updated tool support (VoTE).
  \item A Soundness proof of the associated approximation technique used for 
        this purpose.
  \item An improved node selection strategy that yields speed
        improvements in the range of 4\,\%--91\,\% 
        for trees with a depth of $10$ or more in the two demonstrated case
        studies.
\end{itemize}

The rest of this paper is structured as follows. Section~\ref{sec:preliminaries}
presents preliminaries on decision trees, tree ensembles, and a couple of 
interesting properties subject to verification. Section~\ref{sec:related-works} 
discusses related works on formal verification and machine learning, and 
Section~\ref{sec:main-contrib} presents our method with our supporting tool 
VoTE. Section~\ref{sec:case-studies} presents applications of our method on two
case studies; a collision detection problem, and a digit recognition problem. 
Finally, Section~\ref{sec:conclusions} concludes the paper and summarizes the 
lessons we learned.

\section{Preliminaries}
\label{sec:preliminaries}
In this section, we present preliminaries on different machine learning models
and their properties that we consider for verification in this paper. 

\subsection{Decision Trees}
In machine learning, decision trees are used as predictive models to capture
statistical properties of a system of interest.
\begin{definition}[Decision Tree]
  A decision tree implements a prediction function $t: X^n \rightarrow \mathbb{R}^m$
  that maps disjoint sets of points $X_i \subset X^n$ to a single output point
  $\vec{y}_i \in \mathbb{R}^m$, i.e.,
  \[
  t(\vec{x}) = 
  \begin{cases} 
    (y_{1,1}, \ldots, y_{1,m}) & \vec{x} \in X_{1} \\
    \hfill \vdots\\
    (y_{k,1}, \ldots, y_{k,m}) & \vec{x} \in X_{k}, \\
  \end{cases}
  \] where $k$ is the number of disjoint sets and
  $X^n = \bigcup\limits_{i=1}^{k} X_{i}$.
\end{definition}
For perfectly balanced binary decision trees, $k=2^d$, where $d$ is the 
tree depth. Each internal node is associated with a decision function that
separates points in the input space from each other, and the leaves define 
output values. The n-dimensional input domain $X^n$ includes elements 
$\vec{x}$ as tuples in which each element ${x}_i$ captures some feature of 
the system of interest as an input variable. The tree structure is evaluated in
a top-down manner, where decision functions determine which path to take towards
the leaves. When a leaf is hit, the output $\vec{y} \in \mathbb{R}^m$ associated 
with the leaf is emitted. Figure~\ref{figure:dd-example} depicts a decision tree
with one decision function ($x \leq 0$) and two outputs (1 and 2).
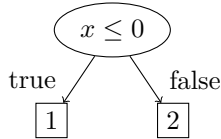
\begin{figure}[ht]
  \centering
  \begin{tikzpicture}[nodes={ellipse,draw}, ->, scale=0.8]
    \tikzstyle{level 1}=[sibling distance=20mm]
    \tikzstyle{every node}=[draw=black, ellipse, align=center]
    \draw
      node{$x \leq 0$}
        child { node[rectangle] {1} 
          edge from parent node[left,draw=none] {true} 
        }
        child { node[rectangle] {2} 
          edge from parent node[right,draw=none] {false}
      };
  \end{tikzpicture}
  \caption{A decision tree with two possible outputs, depending on the value
    of single input variable $x$.}
  \label{figure:dd-example}
\end{figure}

In general, decision functions are defined by non-linear combinations of
several input variables at each internal node. In this paper, we only 
consider binary trees with linear decision functions with one input variable,
which Irsoy et al. call univariate hard decision trees~\cite{Irsoy12}.
As illustrated by Figure~\ref{figure:dd-boxes}, a univariate hard decision tree
forms hyperrectangles (boxes) that split the input space along axes in the
coordinate system.
\begin{figure}[ht]
  \centering
  \begin{tikzpicture}[scale=0.6]
    \tiny
    \draw[->] (0,0) -- (7,0) node[right] {$x_1$}; 
    \draw[->] (0,0) -- (0,5) node[above] {$x_2$};
    \draw (3,0) -- (3,5);
    \draw (3,3) -- (7,3);

    \draw (1.5,2) node[above] {$\vec{x} \in X_1$};
    \draw (5,4) node[above] {$\vec{x} \in X_2$};
    \draw (5,1) node[above] {$\vec{x} \in X_3$};
  \end{tikzpicture}
  \caption{The input space of a univariate hard decision tree, which splits the 
           input space along axes in the coordinate system, thus forming boxes.}
  \label{figure:dd-boxes}
\end{figure}
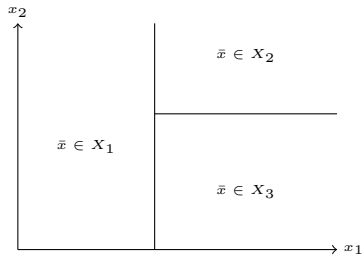

Although researchers have demonstrated that non-linear~\cite{Irsoy12} and
multivariate decision trees~\cite{WangF18} can be useful, state-of-the-art
implementations of tree ensembles normally use hard univariate decision trees,
e.g., scikit-learn~\cite{Pedregosa11} and CatBoost~\cite{Prokhorenkova18}.

\subsection{Random Forests}
\label{sec:random-forest}
Decision trees are known to suffer from a phenomenon called overfitting. Models
suffering from this phenomenon can be fitted so tightly to their training data
that their performance on unseen data is reduced the more you train them. To
counteract these effects in decision trees, Breiman~\cite{Breiman01} proposes
random forests.
\begin{definition}[Random Forest]
\label{def:random-forest}
A random forest implements a prediction function 
$f_{\mathit{rf}}: X^n \rightarrow \mathbb{R}^m$ as the arithmetic mean of the
predictions from $B$ decision trees, i.e.,
$$
f_{\mathit{rf}}(\vec{x}) = \frac{1}{B} \sum\limits_{b=1}^{B} t_b(\vec{x}),
$$ where $t_b$ is the $b$-th tree. 
\end{definition}
To reduce correlation between trees, each tree is trained on a random subset of 
the training data, using potentially overlapping random subsets of the input
variables.

\subsection{Gradient Boosting Machines}
\label{sec:gradient-boosting}
Similarly, Freidman~\cite{Friedman01} introduces a machine learning model called 
gradient boosting machine that uses several decision trees to implement a 
prediction function. Unlike random forests, these trees are trained in a
sequential manner. Each consecutive tree tries to compensate for errors made by
previous trees by estimating the gradient of errors (using gradient decent, 
hence the name). 
\begin{definition}[Gradient Boosting Machine]
\label{def:gradient-boosting}
A gradient boosting machine implements a prediction function 
$f_{\mathit{gb}}: X^n \rightarrow \mathbb{R}^m$ as the sum of predictions from 
$B$ decision trees, i.e.,
$$
f_{\mathit{gb}}(\vec{x}) = \sum\limits_{b=1}^{B} t_b(\vec{x}),
$$ where $t_b$ is the $b$-th tree. 
\end{definition}
Typically, trees in a gradient boosting machine are significantly shallower 
than trees in a random forest, often with a tree depth in the range 2--10.
Gradient boosting machines instead capture complexity by growing more trees.

\subsection{Tree Ensemble}
\label{sec:tree-ensemble}
The learning algorithms used in random forests and gradient boosting machines
are conceptually different from each other. However, once training is completed,
the prediction functions which are the subjects to verification in this paper 
are similar. In a verification context, random forests and gradient boosting
machines can therefore be generalized as instances of a \textit{tree ensemble}.
\begin{definition}[Tree Ensemble]
\label{def:tree-ensemble}
A tree ensemble implements a prediction function $f: X^n \rightarrow \mathbb{R}^m$
as the sum of predictions from $B$ decision trees, post-processed by a function
$p: \mathbb{R}^m \rightarrow \mathbb{R}^m$, i.e.,
$$
f(\vec{x}) = p\Big(\sum\limits_{b=1}^{B} t_b(\vec{x})\Big),
$$ where $t_b$ is the $b$-th tree. 
\end{definition}

In a verification context, the only conceptual difference between a random 
forest and gradient boosting machine is the post-processing function $p$, i.e.,
a constant division $p(\vec{y}) = \frac{\vec{y}}{B}$ in the case of random 
forests, and the identity function $p(\vec{y}) = \vec{y}$ in the case of 
gradient boosting machines.

\subsection{Classifiers}
\label{sec:classifier}
Decision trees and tree ensembles may also be used as classifiers. A 
classifier is a function that categorizes samples from an input domain into
one or more classes. In this paper, we only consider functions that map 
each point from an input domain to exactly one class.
\begin{definition}[Classifier]
\label{def:classifier}
Let $f(\vec{x}) = (y_1, \ldots, y_m)$ be a model trained to predict the 
probability $y_i$ associated with a class $i$ within disjoint regions in the
input domain, where $m$ is the number of classes. Then we would expect that
$\forall i \in \{1, \ldots, m\}, 0 \leq y_i \leq 1$, and 
$\sum\limits_{i=1}^{m} y_i = 1$. A classifier $f_c(\vec{x})$ may then be 
defined as
$$
f_c(\vec{x}) = \argmax_i y_i.
$$
\end{definition}
A random forest typically infers probabilities by capturing the number of 
times a particular class has been observed within some hyperrectangle in 
the input domain of a tree during training. Training a gradient boosting 
machine to predict class membership probabilities is somewhat different,
depending on the characteristics of the used learning algorithm, and often 
involves additional post-processing of the sum of all trees. For example, when 
training multiclass classifiers in CatBoost~\cite{Prokhorenkova18}, individual
trees emit values from a logarithmic domain that are summed up, and finally
transformed and normalized into probabilities using the $\softmax$ function,
i.e.,
\[
  p(\vec{y}) = \softmax(y_1, \ldots, y_m) = \frac{(e^{y_1}, \ldots, e^{y_m})}{\sum\limits_{i=1}^{m}y_i}
\]


\subsection{Safety Properties}
\label{sec:properties}
In this paper, we consider two properties commonly used in related works;
robustness against noise and plausibility of range.\footnote{Other 
works~\cite{Pulina12} refer to this property as ``global safety''. To avoid 
confusion with the dependability term ``safety'', we instead refer to this 
property as ``plausibility of range''.} Note that compliance with 
these two properties alone is generally not sufficient to ensure safety. System
safety engineers typically define requirements on software functions that are 
richer than these two properties alone.

\begin{property}[Robustness against Noise]
  \label{prop:robustness-noise}
  Let $f: X^n \rightarrow \mathbb{R}^m$ be the function subject to verification,
  $\epsilon \in \mathbb{R}_{\ge 0}$ a robustness margin, and
  $\Delta = \{ \delta \in \mathbb{R}: -\epsilon < \delta < \epsilon \}$ noise.
  We denote by $\vec{\delta}$ an $n$-tuple of elements drawn from $\Delta$. 
  The function is robust against noise iff
  $$
  \forall \vec{x} \in X^n, \;\;
  \forall \vec{\delta} \in \Delta^n,
  f(\vec{x}) = f(\vec{x} + \vec{\delta}).
  $$
\end{property}
Pulina and Tacchella~\cite{Pulina12} define a stability property that is similar
to our notion of robustness here but use scalar noise.

\begin{property}[Plausibility of Range]
  \label{prop:plausible-ange}
  Let $f: X^n \rightarrow \mathbb{R}^m$ be the function subject to verification.
  The function has a desired plausibility of range when its output values are 
  within a stated boundary, i.e.,
  $$
  \forall \vec{x} \in X^n,
  \forall i \in \{1, \ldots, m\},
  f(\vec{x}) = (y_1, \ldots, y_m), \;\;
  \alpha_i \leq y_i \leq \beta_i.
  $$ for some $\alpha_i,\beta_i \in \mathbb{R}$.
\end{property}
In classification problems, the output tuple $(y_1, \ldots, y_m)$ contains
probabilities, and thus $\alpha_i = 0$ and $\beta_i = 1$. 

\section{Related Works}
\label{sec:related-works}
Significant progress has been made in the application of machine learning to
autonomous systems, and awareness regarding its security and safety
implications has increased. Researchers from several fields are now addressing 
these problems in their own way, often in collaboration across
fields~\cite{Dagstuhl17a}. For example, Bastani et al.~\cite{Bastani18b} verify
a property called path-specific causal fairness, which is similar to our notion
of robustness. In that work, a method that provides probabilistic guarantees
of fairness is provided, whereas our approach provides formal guarantees.

In the following sections, we group related works in to two categories; those 
that formally verify safety-critical properties of neural networks, and those 
that verify tree-based models.

\subsection{Formal Verification of Neural Networks}
There has been extensive research on formal verification of neural networks.
Pulina and Tacchella~\cite{Pulina12} combine SMT solvers with an
abstraction-refinement technique to analyze neural networks with non-linear
activation functions. They conclude that formal verification of realistically 
sized networks is still an open challenge. Scheibler et al.~\cite{Scheibler15} 
use bounded model checking to verify a non-linear neural network controlling 
an inverted pendulum. They encode the neural network and differential equations
of the system as an SMT formula, and try to verify properties without success. 
These works~\cite{Pulina12,Scheibler15} suggest that SMT solvers are currently 
unable to verify entire non-linear neural networks of realistic sizes.

Huang et al.~\cite{Huang17} present a method to verify the robustness of neural 
network classifiers against perturbations which are specified by the user, e.g.,
change in lightning and axial rotation in images. The method uses polyhedrons 
to capture intermediate perturbations that propagate recursively through the
network, and unlike Pulina and Tacchella~\cite{Pulina12}, operates in a 
layer-by-layer fashion. In each iteration, an SMT solver is used to detect if 
any manipulation captured by a polyhedral would cause a misclassification in 
the output layer. They demonstrate that their method is capable of finding many
adversarial examples, but struggle when the recursion is exhaustive and the
input dimension is large.

Katz et al.~\cite{Katz17} combine the simplex method with a SAT solver to
verify properties of deep neural networks with piecewise linear activation 
functions. They successfully verify domain-specific safety properties of a 
prototype airborne collision avoidance system trained using reinforcement
learning. The verified neural network contains a total of 300 nodes organized 
into 6 layers. Ehlers~\cite{Ehlers17} combines an LP solver with a modified 
SAT solver to verify neural networks. His method includes a technique to 
approximate the overall behavior of the network to reduce the search space 
for the SAT solver. The method is evaluated on two case studies; a collision 
detection problem, and a digit recognition problem. We reuse these two case
studies in our work, and also provide a global approximation of the overall
model (in our cases, tree ensembles).

In \cite{Ivanov18}, Ivanov et al.\ successfully verify safety properties of 
non-linear neural networks trained to approximate closed-loop control systems.
Their approach exploits the fact that the sigmoid function is a solution to a
quadratic differential equation, which enables them to transform sigmoid-based
neural networks into an equivalent non-linear hybrid system. They then 
leverage existing verification tools for hybrid systems to verify the 
reachability property. Even though verification of non-linear hybrid systems
is undecidable in general, existing methods work on many practical examples.
Dutta et al.~\cite{Dutta18} address the plausibility of range property of 
feedback controllers realized by neural networks. Similar to Huang 
et al.~\cite{Huang17}, they capture inputs as polyhedrons, but use mixed 
integer linear programming (MILP) to solve constraints, and demonstrate
practicality on neural networks with thousands of neurons.

Kouvaros and Lomuscio~\cite{Kouvaros18} analyze the robustness of convolutional
neural networks against different kinds of image perturbations, e.g., change in
contrast and scaling. They use MILP to encode the problem as a set of linear
equations which are solved layer-by-layer and evaluate their approach on 
a neural network with 1481 nodes trained on a digit recognition problem.

Tjeng et al.~\cite{Tjeng19} further improve upon previous MILP-based approaches
that verify the robustness of neural networks. In particular, they use linear
programming to improve approximations of inputs to non-linearities throughout
the verification process, and evaluate its contribution together with other
MILP-based optimisation techniques for neural networks available in the
literature.

Wang et.al.~\cite{WangS18} improve upon the work by Ehlers~\cite{Ehlers17} by
using a finer approximation technique, which they also combine with a novel
approach to identify which approximations are too conservative. When such an
overapproximation is identified, they iteratively split and branch the analysis
into smaller pieces. They demonstrate that their contributions improve
performance of several orders of magnitude in several case studies, e.g.,
airborne collision avoidance system, and a digit recognition problem (which we
also address, but for tree ensembles).

Narodytska et al.~\cite{Narodytska18} verify binary neural networks, a class of 
machine learning models that mix floating-point operations with Boolean 
operations, and as such use less memory and are more power efficient during
prediction compared to networks that use floating-point operations exclusively. 
They encode computations that use Boolean operations into a SAT problem, while
floating-point operations are first encoded using MILP, which is then mapped
into the SAT problem. They then use an off-the-shelf SAT-solver to verify 
robustness and equivalence properties. They leverage a counterexample-guided 
search procedure and the structure of neural network to speed up the search 
and demonstrate its effectiveness on the MNIST dataset.
Cheng et al.~\cite{Cheng18} also verify binary neural networks using a SAT-solver,
and present a novel factorisation technique made possible by the fact that weights
in binary neural networks are Boolean-valued. They use state-of-the-art
hardware verification tools to check satisfiability and demonstrate that the
factorisation technique is beneficial on most of the problem instances included
in their evaluation.

Mirman et al.~\cite{Mirman18} use abstract interpretation to verify robustness
of neural networks with convolution and fully connected layers. They evaluate
their method on four image classification problems (one of which we use in our
work), and demonstrate promising performance. 

In this paper, we address similar verification problems as the works mentioned 
above that analyze neural networks. Unfortunately, there is currently no 
established benchmark in the literature that facilitates a direct comparison 
between different machine learning models that takes both validation of, e.g.,
accuracy on a data set, and formal verification of safety-critical requirements
into account.

\subsection{Formal Verification of Decision Trees and Tree Ensembles}
The idea that decision trees may be easier to verify than neural networks is 
demonstrated by Bastani et al.~\cite{Bastani18}. They train a neural network to
play the game Pong, then extract a decision tree policy from the trained neural
network. The extracted tree is significantly easier to verify than the neural
network, which they demonstrate by formally verifying properties within seconds 
using an of-the-shelf SMT solver. Our method provides even better performance 
when verifying decision trees. However, our outlook is that decision trees per
se may not be sufficient for problems in non-trivial settings and hence we 
address tree ensembles which provides a counter-measure to overfitting.

In our previous work~\cite{Tornblom19}, we verify safety-critical 
properties of random forests. Two techniques are presented, a fast but 
approximate technique which yields conservative output bounds, and a slower but
precise technique employed when approximations are too conservative.
In the precise technique, we partition the input space of decision trees 
into disjoint sets, explore all feasible path combinations amongst the trees, 
then compute equivalence classes of the entire random forest. Finally, these
equivalence classes are checked against requirements. In this paper, we
generalize our original method to other tree ensembles such as gradient boosting.
We also improve the performance of the precise technique by changing the
node selection strategy, i.e., the order in which child nodes are considered
while exploring feasible path combinations.

\section{Analyzing Tree Ensembles}
\label{sec:main-contrib}
In this section, we define a process for verifying learning-based systems, and
define a formal method capable of verifying properties of decision trees and
tree ensembles. We also describe VoTE (Verifier of Tree Ensembles) that 
implements our method, and illustrate its usage with an example that verifies
the plausibility of range property of tree ensemble classifiers.

\subsection{Problem Definition}
\label{method:problem}
The software verification process for learning-based systems can be formulated
as the following problem definitions.
\begin{problem}[Constraint Satisfaction]
  Let $f: X^n \rightarrow \mathbb{R}^m$ be a function that is known to implement
  some desirable behavior in a system, and a property $\mathbb{P}$ specifying 
  additional constraints on the relationship between $\vec{x} \in X^n$ and 
  $\vec{y} \in \mathbb{R}^m$. Verify that $\forall \vec{x} \in X^n$, the 
  property $\mathbb{P}$ holds for the computations from $f$.
\end{problem}

Since the prediction function in a tree ensemble is a pure function and thus
there is no state space to explore, this problem may be addressed by considering
all combinations of paths through trees in the ensemble. Furthermore, by 
partitioning the input domain into equivalence classes, i.e., sets of points
in the input space that yield the same output, constraint satisfaction may be
verified for regions in the  input domain, rather than for individual points
explicitly.

\begin{problem}[Equivalence Class Partitioning]
  For each path combination $p$ in a tree ensemble with the prediction function
  $f: X^n \rightarrow \mathbb{R}^m$, determine the complete set of inputs
  $X_p \subseteq X^n$ that lead to traversing $p$, and the corresponding output
  $\vec{y}_p \in \mathbb{R}^m$.
\end{problem}

Our method efficiently generates equivalence classes as pairs of
$(X_p, \vec{y}_p)$, and automatically verifies the satisfaction of a property
$\mathbb{P}$. Assuming that the trees in an ensemble are of equal size, the 
number of path combinations in the tree ensemble is $2^{d \cdot B}$. 
In practice, decisions made by the individual trees are influenced by a subset
of features shared amongst several trees within the same ensemble, and thus 
several path combinations are infeasible and may be discarded from analysis.

\begin{example}[Discarded Path Combination]
  Consider a tree ensemble with the trees depicted in 
  Figure~\ref{figure:rf-example}. There are four path combinations. However, $x$
  cannot be less than or equal to zero at the same time as being greater than
  five. Consequently, Tree 1 cannot emit $1$ at the same time as Tree 2
  emits $3$, and thus one path combination may be discarded from analysis.
  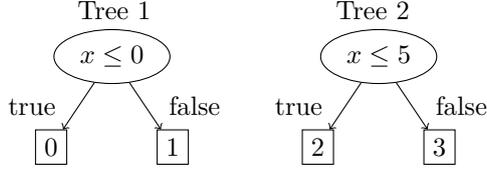
\begin{figure}[ht]
    \centering
    Tree 1 \hspace{2.2cm} Tree 2 \vspace{0.1cm}~\\
    \begin{tikzpicture}[nodes={ellipse,draw}, ->, scale=0.8]
      \tikzstyle{level 1}=[sibling distance=20mm]
      \tikzstyle{every node}=[draw=black, ellipse, align=center]
      \draw
        node{$x \leq 0$}
          child { node[rectangle] {0} 
            edge from parent node[left,draw=none] {true} 
          }
          child { node[rectangle] {1} 
            edge from parent node[right,draw=none] {false}
          };
    \end{tikzpicture}
    \begin{tikzpicture}[nodes={ellipse,draw}, ->, scale=0.8]
      \tikzstyle{level 1}=[sibling distance=20mm]
      \tikzstyle{every node}=[draw=black, ellipse, align=center]
      \draw
        node{$x \leq 5$}
          child { node[rectangle] {2} 
            edge from parent node[left,draw=none] {true} 
          }
          child { node[rectangle] {3} 
            edge from parent node[right,draw=none] {false}
          };
    \end{tikzpicture}
    \caption{Two decision trees that when combined into a tree ensemble,
      contains three feasible path combinations and one discarded path 
      combination.}
    \label{figure:rf-example}
  \end{figure}
\end{example}

We postulate that since several path combinations may be discarded from analysis,
all equivalence classes in a tree ensemble may be computed and enumerated 
within reasonable time for practical applications. To explore this idea, we 
developed the tool VoTE which automates the computation, enumeration, and
verification of equivalence classes.

\subsection{Tool Overview}
VoTE consists of two distinct components, VoTE Core and VoTE Property Checker.
VoTE Core takes as input a tree ensemble with prediction function 
$f: X^n \rightarrow \mathbb{R}^m$, a hyperrectangle defining the input domain
$X^n$ (which may include $\pm\infty$), and computes all equivalence classes in
$f$. These equivalence classes are then processed by VoTE Property Checker that
checks if all input-output mappings captured by each equivalence class are valid
according to a property $\mathbb{P}$, as illustrated by 
Figure~\ref{figure:vote-user}.

\begin{figure}[ht]
  \centering
  \includegraphics[width=9cm]{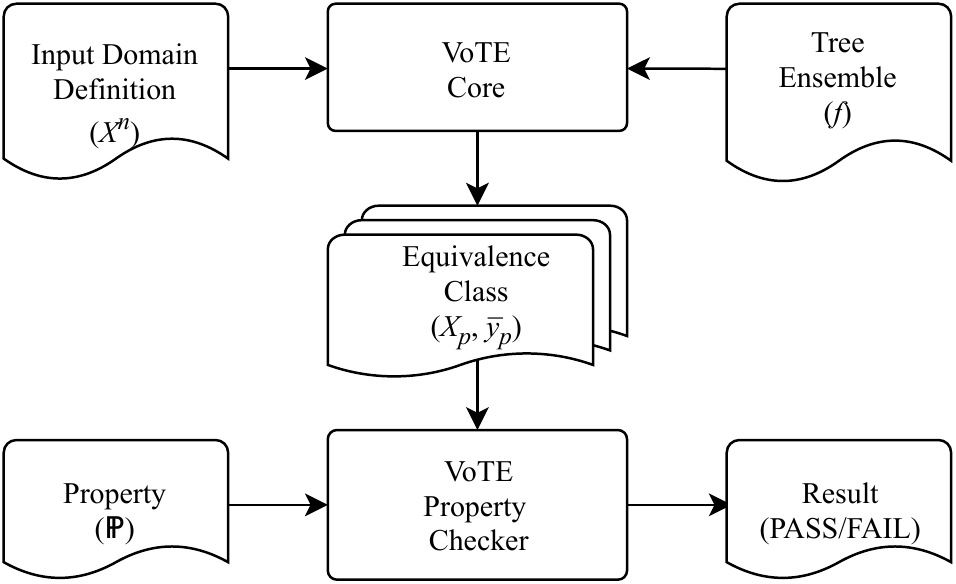}
  \caption{Overview of VoTE.}
  \label{figure:vote-user}
\end{figure}

\subsection{Computing Equivalence Classes}
There are three distinct tasks being carried out by VoTE Core while computing
equivalence classes of a tree ensemble:
\begin{itemize}
  \item partitioning the input domain of decision trees into disjoint sets,
  \item exploring all feasible path combinations in the tree ensemble,
  \item deriving output tuples from leaves.
\end{itemize}
Path exploration is performed by walking the trees depth-first. The order in
which intermediate nodes are considered is described in 
Section~\ref{sec:select-strategy}. When a leaf is hit, the output $\vec{y}_p$ 
for the traversed path combination $p$ is 
incremented with the value associated with the leaf, and path exploration 
continues with the next tree. The set of inputs $X_p$ is captured by a set of
constraints derived from decision functions associated with internal nodes 
encountered while traversing $p$. When all leaves in a path combination have 
been processed, an optional post-processing function is applied to $\vec{y}_p$,
e.g., a division by the number of trees in the case of random forests, and the
$\softmax$ function in the case of a gradient boosting machine classifier
(recall the definition of a random forest in Section~\ref{sec:random-forest}
which includes a division, and the use of the softmax function for
classifiers in Section~\ref{sec:classifier}). Finally, the VoTE Property
Checker checks if the mappings from $X_p$ to $\vec{y}_p$ comply with the
property $\mathbb{P}$. If the property holds, the next available path
combination is traversed, otherwise verification terminates with a ``FAIL'' and
provides the most recent ($X_p$, $\vec{y}_p$) mapping as a counterexample.

\subsection{Node Selection Strategy}
\label{sec:select-strategy}
Each decision function effectively splits the input domain into smaller pieces
throughout the analysis. When a joint evaluation of two decision functions yields
an empty set of points, our method concludes an infeasible path combination and 
continues with the next path combination. One way of improving the performance is
to reduce the time spent on analyzing infeasible path combination by discovering 
them early. Consider the example depicted in Figure~\ref{figure:dd-split-strategy}.
When performing the split as illustrated by the dashed line, the left-hand slice
$X_4$ contains significantly fewer points than the right-hand slice $X_5$. Our
method is based on the idea that by selecting the child nodes in an order based 
on the number of points captured by each slice, splits that yield empty sets of 
points are encountered earlier.
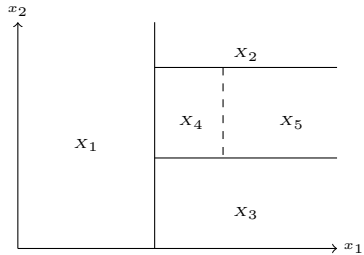
\begin{figure}[ht]
  \centering
  \begin{tikzpicture}[scale=0.6]
    \tiny
    \draw[->] (0,0) -- (7,0) node[right] {$x_1$}; 
    \draw[->] (0,0) -- (0,5) node[above] {$x_2$};
    \draw (3,0) -- (3,5);
    \draw (3,2) -- (7,2);
    \draw (3,4) -- (7,4);
    \draw [dashed] (4.5,4) -- (4.5,2);

    \draw (1.5,2) node[above] {$X_1$};
    \draw (5,4) node[above] {$X_2$};
    \draw (3.8,2.5) node[above] {$X_4$};
    \draw (6,2.5) node[above] {$X_5$};
    \draw (5,0.5) node[above] {$X_3$};
  \end{tikzpicture}
  \caption{An example used to illustrate our node selection strategy. The 
           dashed line indicates a split of a hyperrectangle into two pieces. 
           Our node selection strategy considers the piece with the least 
           number of points first.}
  \label{figure:dd-split-strategy}
\end{figure}

We believe that by choosing the child node which captures the least number of
points first, significant performance is to be expected compared to a static
selection strategy, e.g., by always selecting the left child first as
implemented in previous work~\cite{Tornblom19}.
\subsection{Approximating Output Bounds}
\label{sec:approximating-output-bounds}
The output of a tree ensemble may be bounded by analyzing each leaf in the
collection of trees exactly once. Assuming that all trees are of equal size, 
the number of leaves in a tree ensemble is $B \cdot 2^d$, where $B$ is the 
number of trees and $d$ the tree depth, thus making the analysis scale linearly
with respect to the number of trees.

\begin{definition}[Approximate Tree Output Bounds]
Let $t: X^n \rightarrow \mathbb{R}^m$ be a decision tree with $k$ leaves, and
$T = \{t(\vec{x}): \forall \vec{x} \in X^n\}$ the image of $t$, i.e., the set of 
output tuples associated with those leaves. We then approximate the output of 
$t$ as an interval $[\vec{t}_{\min}, \vec{t}_{\max}]$, where
\begin{gather*}
  \vec{t}_{\min} = (\min\{T_{1, 1}, \ldots, T_{k, 1}\},
                     \ldots,
                     \min\{T_{1, m}, \ldots, T_{k, m}\}), \\
  \vec{t}_{\max} = (\max\{T_{1, 1}, \ldots, T_{k, 1}\},
                     \ldots,
                     \max\{T_{1, m}, \ldots, T_{k, m}\}),
\end{gather*}
and $T_{i, j}$ denotes the $j$-th element in the $i$-th output tuple in $T$.
\end{definition}
\begin{lemma}[Sound Tree Output Approximation]
\label{lemma:sound-tree-output}
The approximate tree output bounds $[\vec{t}_{\min}, \vec{t}_{\max}]$ of a 
decision tree $t: X^n \rightarrow \mathbb{R}^m$ are sound, i.e.,
\[
\forall \vec{x} \in X^n, 
\vec{t}_{\min} \leq t(\vec{x}) \leq \vec{t}_{\max}.
\]
\end{lemma}
\begin{proof}
For an arbitrary $\vec{x} \in X^n$, let $t(\vec{x}) = (v_1, \ldots, v_m)$. 
Expansion of $\vec{t}_{\min} \leq t(\vec{x}) \leq \vec{t}_{\max}$ then yields
\[
\begin{cases}
  \min\{T_{1, 1}, \ldots, T_{k, 1}\} \leq v_1 \leq 
  \max\{T_{1, 1}, \ldots, T_{k, 1}\} \\
  \vdots\\
  \min\{T_{1, m}, \ldots, T_{k, m}\} \leq v_m \leq 
  \max\{T_{1, m}, \ldots, T_{k, m}\}.
\end{cases}
\]
Since $\vec{x}$ is drawn from the domain of $t$, and $T$ is the image of $t$, 
then the output scalar $v_j$ is captured by $T$. Specifically,
$v_j \in \bigcup\limits_{i=1}^{k}\{T_{i,j}\}$.
Hence, as per the definition of the $\max$ and $\min$ set operators,
\[
\forall j \in \{1,\ldots,m\},
  \min\{T_{1, j}, \ldots, T_{k, j}\}  
  \leq v_j \leq 
  \max\{T_{1, j}, \ldots, T_{k, j}\}.
\]
\end{proof}

\begin{definition}[Approximate Ensemble Output Bounds]
Let $f: X^n \rightarrow \mathbb{R}^m$ be an ensemble of $B$ trees with a post
processing function $p: \mathbb{R}^m \rightarrow \mathbb{R}^m$, i.e.,
\[
  f(\vec{x}) = p(t_1(\vec{x}) + \ldots + t_B(\vec{x})),
\] where $t_i: X^n \rightarrow \mathbb{R}^m$ is the $i$-th tree in the ensemble. 
We then approximate the output of $f$ as an interval 
$[\vec{y}_{\min}, \vec{y}_{\max}]$, where
\begin{gather*}
  \vec{y}_{\min} = p(\vec{t}_{\min_1} + \ldots + \vec{t}_{\min_B}), \\
  \vec{y}_{\max} = p(\vec{t}_{\max_1} + \ldots + \vec{t}_{\max_B}),
\end{gather*}
and $[\vec{t}_{\min_i}$, $\vec{t}_{\max_i}]$ is the approximate tree output 
bounds of 
the $i$-th tree.
\end{definition}
\begin{theorem}[Sound Ensemble Output Approximation]
\label{theorem:sound-approx}
The approximate ensemble output bounds \\
$[\vec{y}_{\min}, \vec{y}_{\max}]$ of an
ensemble $f: X^n \rightarrow \mathbb{R}^m$ with a post processing function 
$p: \mathbb{R}^m \rightarrow \mathbb{R}^m$ are sound if $p$ is monotonic, i.e.,
\begin{gather*}
\forall \vec{x} \in X^n, 
\vec{y}_{\min} \leq f(\vec{x}) \leq \vec{y}_{\max}.
\end{gather*}
\end{theorem}
\begin{proof}
\label{proof:sound-approx}
Using Lemma $\ref{lemma:sound-tree-output}$, we know that
$\forall i \in \{1, \ldots, B\}, \forall \vec{x} \in X^n, t_i(\vec{x}) \leq \vec{t}_{\max_i}$, 
and that
\[
t_1(\vec{x}) + \ldots + t_B(\vec{x}) \leq 
\vec{t}_{\max_1} + \ldots + \vec{t}_{\max_B}.
\]
Since $p$ is monotonic, for 
$\vec{v}_1,\vec{v}_2 \in \mathbb{R}^m, \vec{v}_1 \leq \vec{v}_2 \implies p(\vec{v}_1) \leq p(\vec{v}_2)$,
it follows that $\forall \vec{x} \in X^n$,
\begin{gather*}
  t_1(\vec{x}) + \ldots + t_B(\vec{x}) \leq 
  t_{\max_1} + \ldots + t_{\max_B} \implies \\
  p(t_1(\vec{x}) + \ldots + t_B(\vec{x})) \leq 
  p(t_{\max_1} + \ldots + t_{\max_B}) \iff \\
  f(\vec{x}) \leq \vec{y}_{\max}.
\end{gather*}
Analogously, the lower bound $\vec{y}_{\min}$ is also sound.
\end{proof}
These output bounds may be used by a property checker to approximate $f$ in, e.g.,
the plausibility of range property from Section~\ref{sec:properties}, assuming
that the used post-processing function is monotonic. This assumption holds for
random forests since the post-processing function is simply a constant division.
Several prominent gradient boosting machines, e.g., CatBoost~\cite{Prokhorenkova18},
use the $\softmax$ function to post-process multiclass classifications, a function
which is also known to be monotonic~\cite{Gao17}.

Note that this approximation technique is sound, but not complete. If property
checking does not yield ``PASS'' with the approximation (see details below),
the property $\mathbb{P}$ may still hold, and further analysis of the tree
ensemble is required, e.g., by computing all possible equivalence classes (which
is exhaustive and precise). 
\subsection{Implementation}
This section presents implementation details of VoTE Core and VoTE Property 
Checker, aspects that impact accuracy in floating-point computations, and
how VoTE can be adapted to tree ensembles with user-defined post-processing 
functions.
\subsubsection{VoTE Core}
For efficiency, core features in VoTE are implemented as a library in C, and 
utilize a pipeline architecture as illustrated by Figure~\ref{figure:vote-design}
to compute and enumerate equivalence classes.
\begin{figure}[ht]
  \centering
  \includegraphics[width=10cm]{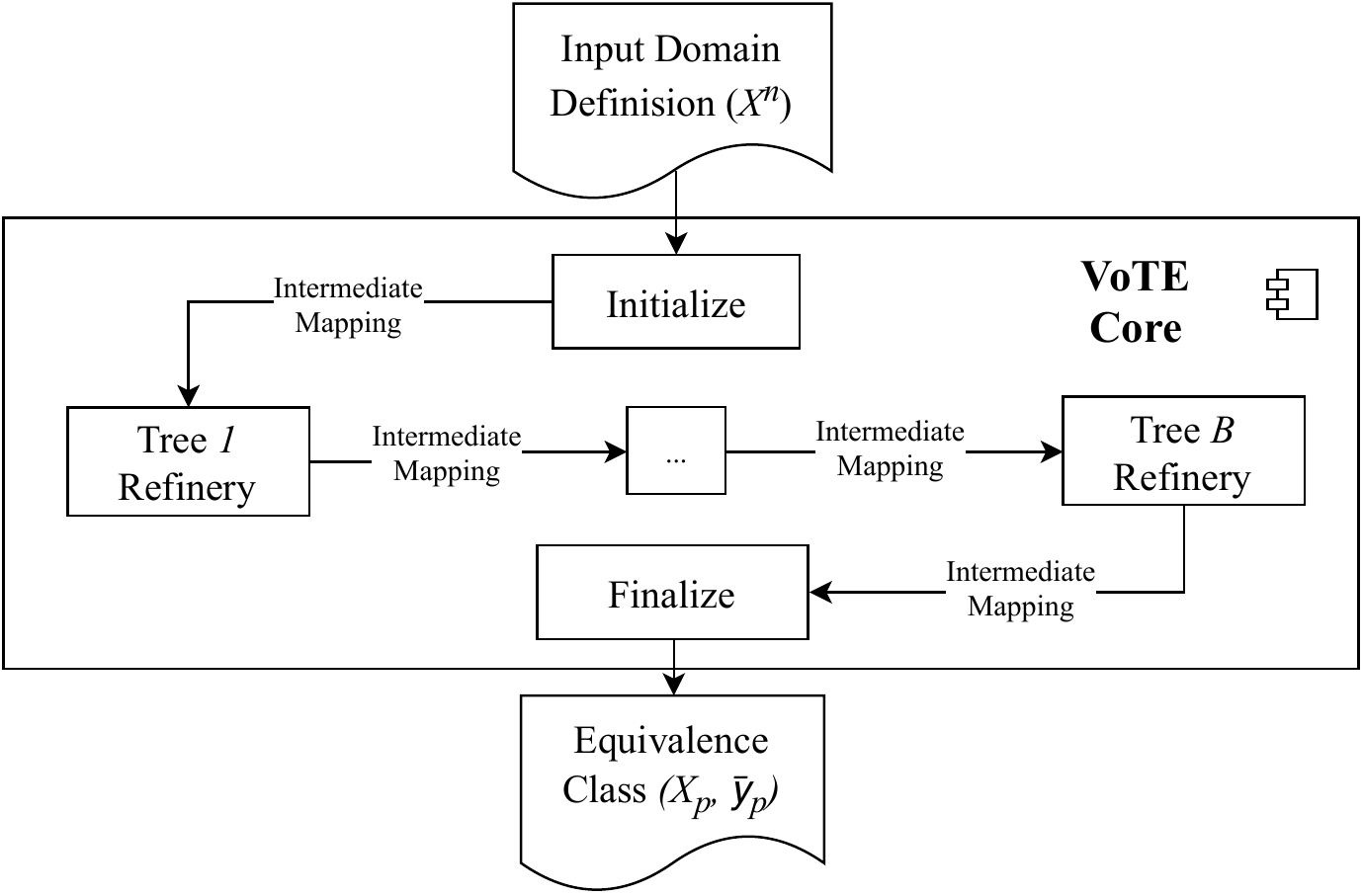}
  \caption{Control flow of equivalence class partitioning in VoTE Core.}
  \label{figure:vote-design}
\end{figure}
The first processing element in the pipeline constructs an intermediate
mapping from the entire input domain to an output tuple of zeros. The final 
processing element applies an optional post-processing function to output
tuples, e.g., a division by the number of trees as in the case of a random 
forest, and the $\softmax$ function in the case of a gradient boosting machine
classifier. In between, there is one refinery element for each tree that splits 
intermediate mappings into disjoint regions according to decision functions in
the trees, and increments the output with values carried by the leaves.

To decouple VoTE from any particular machine learning library, a tree
ensemble is loaded into memory by reading a JSON-formatted file from disk. VoTE 
includes support tools to convert random forests trained by the library
scikit-learn~\cite{Pedregosa11} and gradient boosting machines trained with
CatBoost~\cite{Prokhorenkova18} to this file format.\footnote{The support tools
are published (as free software) at
https://github.com/john-tornblom/vote/tree/v0.1.1/support}

Currently, VoTE Core includes post-processing functions for random forests
and gradient boosting classifiers. Which one to use for a particular problem
instance is specified in the JSON file. VoTE can be easily adapted to
support additional post-processing functions by simply implementing them inside
VoTE Core, or integrating them with a user-defined property checker.


\subsubsection{VoTE Property Checker}
VoTE includes two pre-defined property checkers which are parameterized and
executed from a command-line interface; the plausibility of range property 
checker, and the robustness property checker.

The plausibility of range property checker first uses the output bound 
approximation to check for property violations, and resorts to equivalence-class
analysis only when a violation is detected when using the approximation. 

The robustness property checker checks that all points $X_r$ within a hypercube 
with sides $\epsilon$, centered around a test point $\vec{x}_t$, map to the 
same output. Note that selecting which test points to include in the 
verification may be problematic. In principle, all points in the input domain
should be checked for robustness, but with \emph{classifiers},
there is always a hyperplane separating two classes from each other, and thus
there are always points which violate the robustness property (adjacent to each side of
the hyperplane). Hence, the property is only applicable to points at distances
greater than $\epsilon$ from any classification boundaries.

VoTE also includes Python bindings for easy prototyping of domain-specific 
property checkers. Example~\ref{example:vote-python} depicts an implementation 
of the plausibility of range property that uses these Python bindings to perform 
sanity checking for a classifier's output.
\begin{example}[Plausibility of Range for a Classifier]
  Ensure that the probability of all classes in every prediction is within 
  $[0, 1]$. 
  \label{example:vote-python}
  \lstset{language=Python, morekeywords={assert}}
  \begin{lstlisting}
  import sys
  import vote
  
  def plausibility_of_range(mapping, alpha=0, beta=1):
      minval = min([mapping.outputs[dim].lower
                   for dim in range(mapping.nb_outputs)])
                   
      maxval = max([mapping.outputs[dim].upper
                   for dim in range(mapping.nb_outputs)])
                   
      return (minval >= alpha) and (maxval <= beta)
    
  e = vote.Ensemble(sys.argv[1]) # load model from disk
  assert e.forall(plausibility_of_range)
  \end{lstlisting}
\end{example}

\subsubsection{Computational Accuracy}
Implementations of tree ensembles normally approximate real values as floating-%
point numbers, and thus may suffer from inaccurate computations. In general,
VoTE and the software subject to verification must use the same precision on
floating-point numbers and prediction function in, e.g.,
Definition~\ref{def:random-forest} to get a compatible property satisfaction. 
In this version of VoTE, we use the same representation so that the calculation
errors are the same as in the machine learning library 
scikit-learn~\cite{Pedregosa11} and CatBoost~\cite{Prokhorenkova18}. 
Specifically, we approximate real values as 32-bit floating-point numbers, and 
implement the prediction functions literally as presented in, e.g., 
Definition~\ref{def:random-forest}, i.e., by first computing the sum of all 
individual trees, then dividing by the number of trees. Other machine learning 
libraries may use 64-bit floating-point numbers, and may implement the 
prediction function differently, e.g.,
$$
f(\vec{x}) = \sum \limits_{b=1}^{B} \frac{t_b(\vec{x})}{B}.
$$
This would be easily changeable in VoTE.

\section{Case Studies}
\label{sec:case-studies}
In this section, we present an evaluation of VoTE on two case studies from the 
literature where neural networks have been analyzed for compliance with 
interesting properties. Each case study defines a training set and a test set.
We used scikit-learn~\cite{Pedregosa11} to train random forests, and 
CatBoost~\cite{Prokhorenkova18} to train gradient boosting machines. For random
forests, all training parameters except the number of trees and maximum tree 
depth were kept constant and at their default values. When training gradient 
boosting machines, we also adjusted the learning rate to 0.5 since the default
value demonstrated poor accuracy on our case studies. Furthermore, since
gradient boosting machines typically use shallower trees than random forests,
we used different tree depths and different number of trees for these types of ensembles. 
In fact, CatBoost is limited to a maximum tree depth of 16.

We evaluated accuracy on each trained model against its test set, i.e., the
percentage of samples from the test set where there are no misclassifications,
in order to ensure that we were verifying instances that were interesting 
enough to evaluate. We then developed verification cases for the plausibility 
of range and robustness against noise properties (from 
Section~\ref{sec:properties}) using VoTE. The time spent on verification was
recorded for each trained model as presented below. Next, we evaluated the 
least-points-first node selection strategy (from 
Section~\ref{sec:select-strategy}) against two baselines on all case studies 
(always picking the left child first, and always picking the right child first).

Experiments were conducted on a single machine with an Intel Core i5 2500\,K CPU
and 16\,GB RAM. Furthermore, we used a GeForce GTX 1050 Ti GPU with 4\,GB of memory
to speed up training of gradient boosting machines.

\subsection{Vehicle Collision Detection}
\label{sec:vehicle-detection}
In this case study, we verified properties of tree ensembles trained to detect
collisions between two moving vehicles traveling along curved trajectories at
different speeds. Each verified model accepts six input variables, emits
two output variables, and contains 20--25 trees with depths between 5--20.

\subsubsection{Dataset}
We used a simulation tool from Ehlers~\cite{Ehlers17} to generate 30,000
training samples and 3,000 test samples. Unlike neural networks which Ehlers
used in his case study, the size of a tree ensemble is limited by the amount of
data available during training. Hence, we generated ten times more training data
than Ehlers to ensure that sufficient data is available for the size and number
of trees assessed in our case study. Each sample contains the relative distance
between the two vehicles, the speed and starting direction of the second 
vehicle, and the rotation speed of both vehicles. Each feature in the dataset
is given in normalized form (position, speed, and direction fall in the range
$[0, 1]$, and rotation speed in the range $[-1, 1]$).

\subsubsection{Robustness}
We verified the robustness against noise for all trained models by defining
input regions surrounding each sample in the test set with the robustness margin
$\epsilon = 0.05$, which amounts to a 5\,\% change since the data is normalized. 
Table~\ref{tbl:vcd-robusteness} lists tree ensembles included in the experiment
with their maximum tree depth $d$, number of trees $B$, accuracy of the
classifications (Accuracy), elapsed time $T$ during verification, and the 
percentage of samples from the test set where there were no misclassifications 
within the robustness region (Robustness).

\begin{table}[ht]
  \centering
  \caption{Accuracy, robustness, and elapsed verification time (T) of tree
           ensembles in the vehicle collision detection case study.}
  \label{tbl:vcd-robusteness}
  \begin{tabular}{crrccr}
    \hline
    Type & \mcc{$d$} & \mcc{$B$} & Accuracy\,(\%) & Robustness\,(\%) & $T$\,(s) \\
    \hline
    RF   & 10        & 20        & 90.4           & 48.9             & 56       \\
    RF   & 10        & 25        & 90.0           & \textbf{50.3}    & 286      \\
    RF   & 15        & 20        & 93.0           & 34.1             & 273      \\
    RF   & 15        & 25        & 92.9           & 35.1             & 1651     \\
    RF   & 20        & 20        & 94.2           & 29.5             & 367      \\
    RF   & 20        & 25        & \textbf{94.5}  & 29.6             & 2520     \\
    \hdashline
    GB   & 5         & 20        & 93.4           & \textbf{44.5}    & 1        \\
    GB   & 5         & 25        & 93.8           & 40.4             & 2        \\    
    GB   & 10        & 20        & 95.5           & 34.4             & 26       \\
    GB   & 10        & 25        & 95.7           & 34.0             & 69       \\
    GB   & 15        & 20        & 95.8           & 34.0             & 222      \\
    GB   & 15        & 25        & \textbf{96.0}  & 33.8             & 511      \\
    \hline
  \end{tabular}
\end{table}

Increasing the maximum depth of trees increased accuracy on the test set, but
reduced the robustness against noise. Adding more trees to a random forest
slightly improves its robustness, while gradient boosting machines decreased
their robustness against noise as more trees were added. These observations 
suggest that the models were over-fitted with noiseless examples during 
training, and thus adding noisy examples to the training set may improve 
robustness. The elapsed time during verification was significantly less for 
gradient boosting machines than random forests (using the same parameters). 
The significant difference in elapsed time between, e.g., gradient boosting 
machines with $\{d=5, B=20\}$ and $\{d=15, B=25\}$, may seem counter-intuitive
at first. However, recall that the theoretical upper limit of the number of 
path combinations in a tree ensemble is $2^{d \cdot B}$, and that 
$2^{5 \cdot 20} \ll 2^{15 \cdot 25}$.

Since our observations suggest that the models were over-fitted, we generated a
new training data set which contains 750,000 additional samples with additive 
noise drawn from the uniform distribution. Specifically, for each sample in the 
noiseless training data set with an input tuple $\vec{x}$, we added 25 new 
samples with the input $\vec{x} + \vec{z}$ and the same output, where $\vec{z}$
is a tuple of elements drawn from $\mathcal{U}(-\epsilon, \epsilon)$. We then 
reran the experiments after training new models on the noisy data set. The
results from these experiments are listed in
Table~\ref{tbl:vcd-robusteness-noisy} in the same format as before.
\begin{table}[ht]
  \centering
  \caption{Accuracy, robustness, and elapsed verification time (T) of tree
           ensembles in the vehicle collision detection case study, trained
           on a data set with additive random noise.}
  \label{tbl:vcd-robusteness-noisy}
  \begin{tabular}{crrccr}
    \hline
    Type & \mcc{$d$} & \mcc{$B$} & Accuracy\,(\%) & Robustness\,(\%) & $T$\,(s) \\
    \hline
    RF   & 10        & 20        & 89.3           & 59.4             & 209      \\
    RF   & 10        & 25        & 89.1           & \textbf{60.0}    & 904      \\
    RF   & 15        & 20        & 92.4           & 42.4             & 2779     \\
    RF   & 15        & 25        & 92.2           & 42.4             & 12833    \\
    RF   & 20        & 20        & 93.6           & 31.3             & 7640     \\
    RF   & 20        & 25        & \textbf{93.7}  & 32.7             & 53387    \\
    \hdashline
    GB   & 5         & 20        & 92.7           & \textbf{50.7}    & 2        \\
    GB   & 5         & 25        & 93.1           & 46.2             & 4        \\
    GB   & 10        & 20        & 94.6           & 35.7             & 44       \\
    GB   & 10        & 25        & 94.7           & 34.2             & 142      \\
    GB   & 15        & 20        & 94.6           & 33.3             & 391      \\
    GB   & 15        & 25        & \textbf{94.8}  & 32.5             & 1642     \\
    \hline
  \end{tabular}
\end{table}

Adding noise to the training data reduced the accuracy with at most 1.2\,\%, but
improved on the robustness against perturbations with up to 10.5\,\%. More 
interestingly, models trained on noisy data were more time-consuming to verify
compared to those trained on noiseless data, particularly when using random
forests. After careful inspection of trees trained on the different data sets,
we noticed that random forests trained on the noiseless data contain branches
that are shorter than the maximum tree depth. Consequently, there are fewer path
combinations in these random forests compared to those trained on the noisy data,
which could explain the differing measurements in elapsed verification times.

\subsubsection{Node Selection Strategy}
Next, we evaluated the least-points-first node selection strategy against the 
two baseline strategies. Table~\ref{tbl:vcd-node-select} lists the elapsed 
verification time for the evaluated models when using the least-points-first 
node selection strategy ($T$), always selecting the left child 
first as implemented in the tool VoRF from previous work~\cite{Tornblom19} 
($T_\mathit{VoRF}$), and always selecting the right child first ($T_\mathit{right}$).
\begin{table}[ht]
  \centering
  \caption{Elapsed time for different node selection strategies in the
           vehicle collision detection case study.}
  \label{tbl:vcd-node-select}
  \begin{tabular}{crrrrr}
    \hline
    Type & \mcc{$d$} & \mcc{$B$} & $T_\mathit{VoRF}$\,(s) & $T_\mathit{right}$\,(s) & $T$\,(s) \\
    \hline
    RF   & 10        & 20        & 79             & 74             & 56       \\
    RF   & 10        & 25        & 422            & 374            & 286      \\
    RF   & 15        & 20        & 399            & 441            & 273      \\
    RF   & 15        & 25        & 2351           & 2457           & 1651     \\
    RF   & 20        & 20        & 930            & 847            & 367      \\
    RF   & 20        & 25        & 5499           & 4522           & 2520     \\
    \hdashline
    GB   & 5         & 20        & 1              & 1              & 1        \\
    GB   & 5         & 25        & 3              & 3              & 2        \\
    GB   & 10        & 20        & 30             & 31             & 26       \\
    GB   & 10        & 25        & 84             & 85             & 69       \\
    GB   & 15        & 20        & 265            & 259            & 222      \\
    GB   & 15        & 25        & 618            & 616            & 511      \\
    \hline
  \end{tabular}
\end{table}
The least-points-first node selection strategy was more effective on random 
forests than on gradient boosting machines, with speedup factors in the range 
1.3--2.5 versus 1.0--1.5, respectively. However, gradient boosting machines 
were already significantly easier to verify than random forests (with the same 
number of trees and depth).

\subsubsection{Scalability}
\label{sec:vehicle-scalability}
Next, we assessed the scalability of VoTE Core when the number of trees grows by
verifying the trivial property $\mathbb{P} = \mathit{true}$ which accepts all
input-output mappings. We implemented this trivial property in a verification
case that also counts the number of equivalence classes emitted by VoTE Core. We
then executed the verification case for models trained with a maximum tree depth
of $d = 10$. The recorded number of equivalence classes $C$ for different number
of trees $B$ is depicted in Figure~\ref{figure:vcd-scalability} on a logarithmic
scale.
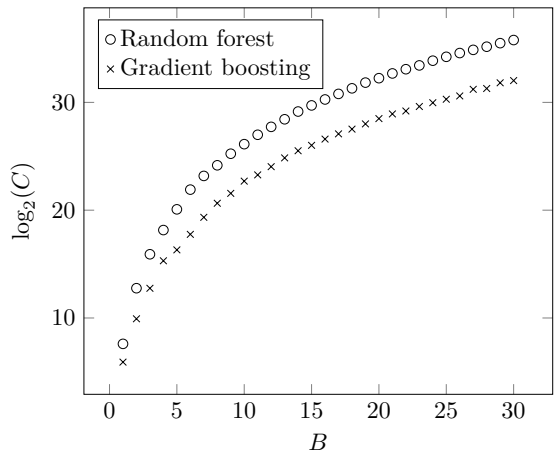
\begin{figure}[ht]
  \centering
  \begin{tikzpicture}[scale=0.9]
    \begin{axis}[
      xlabel={$B$},
      ylabel={$\log_2(C)$},
      legend cell align={left},
      legend pos={north west},
      legend entries={Random forest, Gradient boosting}]
      \addplot[scatter src=explicit, only marks, mark=o] table[
          x expr=\thisrowno{0}, 
          y expr=log2(\thisrowno{1}) ]
       {
         B  Cardinality
         1  194
         2  6947
         3  61562
         4  292112
         5  1107049
         6  3935246
         7  9491897
         8  18797511
         9  39563182
         10 73177470
         11 133862181
         12 222632459
         13 362135199
         14 596223534
         15 883465913
         16 1294736880
         17 1852694610
         18 2651408741
         19 3829488642
         20 5090515209
         21 6905759122
         22 9058250244
         23 11574497153
         24 15676619289
         25 20113003227
         26 25619263227
         27 31446269326
         28 38510754321
         29 48385071382
         30 59117959084
       };
       \addplot[scatter src=explicit, only marks, mark=x] table[
          x expr=\thisrowno{0}, 
          y expr=log2(\thisrowno{1}) ]
       {
         B  Cardinality
         1  60
         2  972
         3  6912
         4  40500
         5  81648
         6  221760
         7  662400
         8  1638000
         9  3093552
         10 6791400
         11 10138590
         12 17203200
         13 30250584
         14 47880000
         15 67815902
         16 101347200
         17 141796200
         18 192535200
         19 269233216
         20 378675000
         21 510898752
         22 617569920
         23 823987440
         24 1050726600
         25 1314201600
         26 1624613760
         27 2487270240
         28 2621356920
         29 3782643072
         30 4361280000
       };

    \end{axis}
  \end{tikzpicture}
  \caption{Number of equivalence classes $C$ on a logarithmic scale from the 
  vehicle collision detection case study for different number of trees $B$ with
  a depth $d=10$.}
  \label{figure:vcd-scalability}
\end{figure}
The number of equivalence classes increased exponentially as more trees were 
added, but the magnitude of the growth decreased for each added tree. The number
of equivalence classes for large number of trees are significantly smaller than
the upper limit of $2^{d \cdot B}$ (which occurs when there are no shared 
features amongst trees, and thus each path combination yields a distinct
equivalence class). Furthermore, the gradient boosting machines consistently
yield significantly fewer equivalence classes than random forests, which
could explain the differences in verification times we observed between the 
two types of models (with the same number of trees and depth).

\subsubsection{Plausibility of Range}
Finally, we verified the plausibility of range property (here ensuring that all 
predicted probabilities are in the range $[0, 1]$). All random forests passed 
the verification case within fractions of a second thanks to the fast output 
bound approximation algorithm. For gradient boosting machines however, the 
output approximations were too conservative; hence we resorted the precise 
technique. All gradient boosting machines passed the verification case, and 
the elapsed time during verification for different node selection strategies 
are listed in Table~\ref{tbl:vcd-plausible-range}.

\begin{table}[ht]
  \centering
  \caption{Elapsed time for different node selection strategies when
           verifying gradient boosting machines in the vehicle collision 
           detection case study.}
  \label{tbl:vcd-plausible-range}
  \begin{tabular}{rrrrr}
    \hline
    \mcc{$d$} & \mcc{$B$} & $T_\mathit{VoRF}$\,(s) & $T_\mathit{right}$\,(s) & $T_\mathit{VoTE}$\,(s) \\
    \hline
    5         & 20        & 2              & 2              & 2        \\
    5         & 25        & 10             & 10             & 9        \\
    10        & 20        & 304            & 289            & 260      \\
    10        & 25        & 1067           & 1015           & 917      \\
    15        & 20        & 3032           & 3230           & 2910     \\
    15        & 25        & 10582          & 9900           & 8750     \\
    \hline
  \end{tabular}
\end{table}
The least-points-first node selection strategy consistently outperformed the two
baseline strategies, with speedup factors in the range 1.0--1.2.

\subsection{Digit Recognition}
In this case study, we verified properties of tree ensembles trained to
recognize images of hand-written digits.

\subsubsection{Dataset}
The MNIST dataset~\cite{MNIST} is a collection of hand-written digits commonly
used to evaluate machine learning algorithms. The dataset contains 70,000
gray-scale images with a resolution of $28 \times 28$ pixels at 8\,bpp, encoded
as tuples of 784 scalars. We randomized the dataset and split into two subsets;
a 85\,\% training set, and a 15\,\% test set (a similar split was used
in~\cite{MNIST}).

\subsubsection{Robustness}
\label{sec:DigitRecognition-Robustness}
We defined input regions surrounding each sample in the test set with the 
robustness margin $\epsilon = 1$, which amounts to a 0.5\,\% lighting change per
pixel in a 8\,bpp gray-scaled image. Each input region contains $2^{784}$ noisy 
images, which would be too many for VoTE to handle within a reasonable amount
of time. Consequently, we reduced the complexity of the problem significantly by
only considering robustness against noise within a sliding window of
$5 \times 5$ pixels. For a given sample from the test set, noise was added within
the $5 \times 5$ window, yielding $2^{5 \cdot 5}$ noisy images. This operation was
then repeated on the original image, but with the window placed at an offset of
1px relative to its previous position. Applying this operation on an entire
image yields $2^{5 \cdot 5} \cdot (28-5)^2 \approx 2^{34}$ distinct noisy images per
sample from the test set, and about $10^{14}$ noisy images when applied to the
entire test set.

Figure~\ref{figure:mnist-missclass} depicts one of many examples from the MNIST 
dataset that were misclassified by the tree ensemble with $B = 25$ and $d = 10$. 
Since the added noise is invisible to the naked eye, the noise (a single pixel) 
is highlighted in red inside the circle.
\begin{figure}[ht]
  \centering
  \includegraphics[width=3cm]{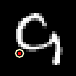}
  \caption{A misclassified noisy sample from the MNIST dataset.}
  \label{figure:mnist-missclass}
\end{figure}

Table~\ref{tbl:mnist-robusteness} lists tree ensembles included in the
experiment with their maximum tree depth $d$, number of trees $B$, accuracy on
the test set (Accuracy), elapsed time $T$ during verification, and the
percentage of samples from the test set where there were no misclassifications
within the robustness region (Robustness).
\begin{table}[ht]
  \centering
  \caption{Accuracy and robustness of tree ensembles in the digit recognition
           case study.}
  \label{tbl:mnist-robusteness}
  \begin{tabular}{crrccr}
    \hline
    Type & \mcc{$d$} & \mcc{$B$} & Accuracy\,(\%) & Robustness\,(\%) & $T$\,(s) \\
    \hline
    RF   & 10        & 20        & 93.8           & 75.2             & 254      \\
    RF   & 10        & 25        & 94.2           & 74.8             & 1217     \\
    RF   & 15        & 20        & 95.8           & 82.8             & 436      \\
    RF   & 15        & 25        & 96.0           & \textbf{84.0}    & 2141     \\
    RF   & 20        & 20        & 96.0           & 82.3             & 391      \\
    RF   & 20        & 25        & \textbf{96.4}  & 83.7             & 1552     \\
    \hdashline
    GB   & 5         & 75        & 94.5           & 60.9             & 129      \\
    GB   & 5         & 150       & 95.3           & 67.3             & 301      \\
    GB   & 5         & 300       & 95.7           & 68.6             & 1551     \\
    GB   & 10        & 25        & 94.9           & 65.8             & 82       \\
    GB   & 10        & 50        & 95.7           & 73.9             & 159      \\
    GB   & 10        & 100       & \textbf{96.3}  & \textbf{78.8}    & 486      \\
    \hline
  \end{tabular}
\end{table}

Increasing the complexity of a tree ensemble slightly increased its accuracy,
and significantly increased its robustness against noise. The elapsed time 
during verification was significantly less for gradient boosting machines than
random forests (using the same parameters).

\subsubsection{Node Selection Strategy}
Next, we evaluated the least-points-first node selection strategy against the 
two baseline strategies. Table~\ref{tbl:mnist-node-select} lists the elapsed
verification time for the evaluated models when using the least-points-first
node selection strategy ($T$), always selecting the left child 
first as implemented in the tool VoRF from previous work~\cite{Tornblom19} 
($T_\mathit{VoRF}$), and always selecting the right child first
($T_\mathit{right}$).

\begin{table}[ht]
  \centering
  \caption{Elapsed time for different child node selection strategies in the
           digit recognition case study.}
  \label{tbl:mnist-node-select}
  \begin{tabular}{crrrrr}
    \hline
    Type & \mcc{$d$} & \mcc{$B$} & $T_\mathit{VoRF}$\,(s) & $T_\mathit{right}$\,(s) & $T$\,(s) \\
    \hline
    RF   & 10        & 20        & 2009            & 1093          & 254      \\
    RF   & 10        & 25        & 10724           & 5386          & 1217     \\
    RF   & 15        & 20        & 4474            & 1837          & 436      \\
    RF   & 15        & 25        & 23718           & 8960          & 2141     \\
    RF   & 20        & 20        & 4037            & 1817          & 391      \\
    RF   & 20        & 25        & 17360           & 7228          & 1552     \\
    \hdashline
    GB   & 5         & 75        & 836            & 376            & 129      \\
    GB   & 5         & 150       & 2419           & 848            & 301      \\
    GB   & 5         & 300       & 18829          & 4717           & 1551     \\
    GB   & 10        & 25        & 442            & 268            & 82       \\
    GB   & 10        & 50        & 1157           & 561            & 159      \\
    GB   & 10        & 100       & 5772           & 1618           & 486      \\
    \hline
  \end{tabular}
\end{table}
The effectiveness of our node selection strategy was similar for both random
forests and gradient boosting machines, with significant speed up factors in 
the range 4.2--11.2 and 2.8--12.1, respectively.

\subsubsection{Scalability}
Next, we assessed the scalability of VoTE Core when the number of trees grows 
by verifying the trivial property $\mathbb{P} = \mathit{true}$. This was done in a 
similar way as described in the vehicle collision detection use case presented 
in Section~\ref{sec:vehicle-scalability}. We then executed the verification case
for all models with a tree depth of $d = 10$. Enumerating all possible 
equivalence classes was intractable for tree ensembles with more than $B = 4$ 
trees. We aborted the experiment after running the verification case with a 
tree ensembles of $B = 5$ for 72\,h. Figure~\ref{figure:mnist-scalability} 
depicts the four data points we managed to acquire for the two types of models.
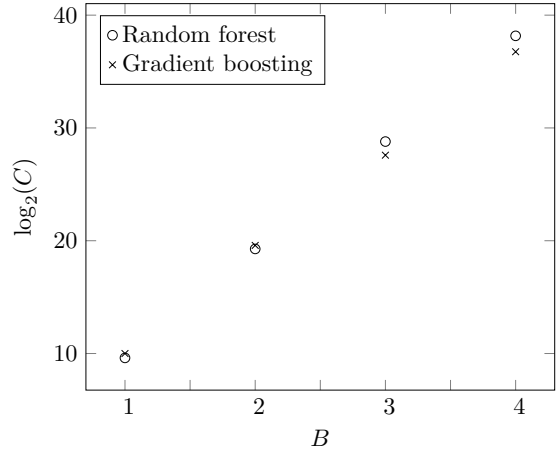
\begin{figure}[ht]
  \centering
  \begin{tikzpicture}[scale=0.9]
    \begin{axis}[
      xlabel={$B$},
      xticklabel={
        \pgfmathtruncatemacro{\IntegerTick}{\tick}
        \pgfmathprintnumberto[verbatim,fixed,precision=3]{\tick}\tickAdjusted
        \pgfmathparse{\IntegerTick == \tickAdjusted ? 1: 0}
        \ifnum\pgfmathresult>0\relax$\IntegerTick$\else\fi
      },  
      legend cell align={left},
      legend pos={north west},
      legend entries={Random forest, Gradient boosting},
      ylabel={$\log_2(C)$}]
      \addplot[scatter src=explicit, only marks,mark=o] table [
          x expr=\thisrowno{0}, 
          y expr=log2(\thisrowno{1})
      ]
       {
         B  Cardinality
         1  782
         2  633172
         3  464878397
         4  308175596371
      };
      \addplot[scatter src=explicit, only marks, mark=x] table[
          x expr=\thisrowno{0}, 
          y expr=log2(\thisrowno{1})
      ]
       {
         B  Cardinality
         1  1024
         2  786432
         3  201326592
         4  115964116992
      };
    \end{axis}
  \end{tikzpicture}
  \caption{Number of equivalence classes $C$ on a logarithmic scale from the 
  digit recognition case study for different number of trees $B$ with
  a depth $d=10$.}
  \label{figure:mnist-scalability}
\end{figure}

The number of equivalence classes increased exponentially as more trees were 
added, without demonstrating any signs of stagnation. The ability to discard 
infeasible path combinations in a tree ensembles is an essential ingredient to
our method. When tree ensembles are trained on high-dimensional data, the number
of features shared between trees is relatively low, so it is not surprising that
our method experiences combinatorial path explosion. 

As described in Section~\ref{sec:DigitRecognition-Robustness}, state explosion 
was already anticipated when considering robustness to noise when changing 
arbitrary pixels in the whole state space. This was the underlying reason why
the property $\mathbb{P}$ to verify was formulated as robustness to noise when
changing pixels within a sliding window of $5 \times 5$ pixels, which 
significantly reduced the search space. This reduction in space was anticipated
based on some intuition about the application domain. Not having this intuition
may lead to trying to prove properties that are tougher than required, or 
eliminating a class of applications (e.g., image processing) with long 
verification times.

\subsubsection{Plausibility of Range}
Finally, we verified the plausibility of range property (again ensuring that 
that all predicted probabilities are in the range $[0, 1]$). All random forests
passed the verification case within fractions of a second thanks to the fast 
output bound approximation algorithm. For gradient boosting machines however,
the output approximations were too conservative. Since the precise technique 
does not scale well on models trained on high-dimensional data, we were unable
to verify the plausibility of range property of gradient boosting machines in 
this case study.

\section{Conclusions and Future Work}
\label{sec:conclusions}
In this paper, we proposed a method to formally verify properties of tree 
ensembles. Our method exploits the fact that several trees make decisions based
on a shared subset of the input variables, and thus several path combinations in
tree ensembles are infeasible. We implemented the method in a tool called VoTE, 
and demonstrated its scalability on two case studies. 

In the first case study, a collision detection problem with six input
variables, we demonstrated that problems with a low-dimensional input space can
be verified using our method within a reasonable amount of time. In the second 
case study, a digit recognition problem with 784 input variables, we 
demonstrated that our method copes with high-dimensional input space when
verifying robustness against noise. But it does so only if the systematically
introduced noise does not attempt to exhaustively cover all possibilities.
Since the number of shared input variables between trees is low, we observed a
combinatorial explosion of paths in the tree ensembles. This combinatorial
explosion also appeared when we verified the plausibility of range property of
gradient boosting machines where the fast approximation technique was too 
conservative. However, when verifying the plausibility of range property of 
random forests, the approximation technique was sufficiently accurate, and 
verification was completed within seconds.

There exist properties and systems for which formal verification may not be
successfully deployed. Conversely, using random sampling for testing robustness
may not be acceptable as an argument in a safety assurance case, unless a clear
relation between the test samples and the property needed for safety can be
mathematically established. Additional work to identify where to apply testing
and where formal verification saves time is a useful direction to explore.

For future work, we plan to investigate different tree selection strategies, 
i.e., strategies that determine in which order trees in an ensemble are analyzed.
We also consider combining our approximation technique with our precise
technique into an abstraction-refinement scheme. Other directions 
of work include creating new properties that are meaningful in the context of 
the problem at hand, e.g., decisive classifications, and applying to use cases 
where control is involved (and not only sensing).

\subsection*{Acknowledgements}
This work was partially supported by the Wallenberg AI, Autonomous Systems and 
Software Program (WASP) funded by the Knut and Alice Wallenberg Foundation.


\bibliographystyle{model1b-num-names}
\bibliography{19vote}

\end{document}